\documentclass[12pt,a4paper,oneside,notitlepage]{article}

\usepackage{c-series}

\usepackage{url}
\usepackage{pstricks}
\usepackage{booktabs}
\usepackage{graphicx}
\usepackage{url}
\usepackage{amsfonts}
\usepackage{amsthm}
\usepackage{amsmath}
\usepackage{comment}
\theoremstyle{plain}

\newtheorem{proposition}{Proposition}

\theoremstyle{definition}
\newtheorem{definition}{Definition}

\newtheorem{problem}{Problem}

\theoremstyle{remark}

\newcommand{\Abs}[1]{\left|#1\right|}

\newcommand{\Set}[1]{\left\{#1\right\}}

\newcommand{\Oh}[1]{\mathcal{O}(#1)}

\newcommand{\HVP}{\textsc{HVP}}
\newcommand{\HSP}{\textsc{HSP}}

\newcommand{\phase}{\textsc{Phase}}
\newcommand{\fastphase}{fast\textsc{Phase}}
\newcommand{\haplorec}{\textsc{HaploRec}}
\newcommand{\gerbil}{\textsc{Gerbil}}
\newcommand{\spamm}{\textsc{SpaMM}}
\newcommand{\hit}{\textsc{HIT}}
\newcommand{\phases}{\textsc{P}}
\newcommand{\fastphases}{f\textsc{P}}
\newcommand{\haplorecs}{\textsc{HR}}
\newcommand{\gerbils}{\textsc{G}}
\newcommand{\spamms}{\textsc{S}}
\newcommand{\hits}{\textsc{HIT}}
\newcommand{\hap}{\textsc{HAP}}
\newcommand{\hint}{\textsc{HINT}}
\newcommand{\gevalt}{\textsc{Gevalt}}
\newcommand{\trio}{\textsc{trio}}

\newcommand{\argminx}[1]{\underset{#1}{\mathop{argmin}}}
\newcommand{\minx}[1]{\underset{#1}{\mathop{min}}}
\newcommand{\argmaxx}[1]{\underset{#1}{\mathop{argmax}}}
\newcommand{\argmax}{\mathop{argmax}}
\newcommand{\argmin}{\mathop{argmin}}

\begin{document}

\title{Combining haplotypers}
\titlelinebreaks{Combining haplotypers}

\author{Matti K{\"a}{\"a}ri{\"a}inen
\and
Niels Landwehr
\and
Sampsa Lappalainen
\and
Taneli Mielik{\"a}inen}

\authorcontact{Matti K{\"a}{\"a}ri{\"a}inen \\
HIIT Basic Research Unit \\
Department of Computer Science\\
University of Helsinki \\
\texttt{matti.kaariainen@cs.helsinki.fi} \\
\ \\
\and
Niels Landwehr\\
Machine Learning Lab\\
Department of Computer Science \\
Katholieke Universiteit Leuven\\
\texttt{niels.landwehr@cs.kuleuven.be} \\
\ \\
\and
Sampsa Lappalainen\\
HIIT Basic Research Unit \\
Department of Computer Science\\
University of Helsinki \\
\texttt{sampsa.lappalainen@cs.helsinki.fi} \\
\ \\
\and
Taneli Mielik{\"a}inen\\
Department of Computer Science\\
University of Helsinki and \\
Nokia Research Center Palo Alto \\
\texttt{taneli.mielikainen@iki.fi} \\
\ \\
}

\reportyear{2007}
\reportno{57}
\reportmonth{September}
\CRClasses{F.2.2, I.2.6, J.3}
\printhouse{}
\reportpages{\pageref{contentsendpage} + \pageref{lastpage}}
\newpage

\CRClassesLong{\newlength{\restline}
\setlength{\restline}{\textwidth}
\addtolength{\restline}{-13mm}
\parbox[t]{13mm}{F.2.2}
\vspace{1ex}
\parbox[t]{\restline}{Analysis of Algorithms and Problem Complexity: Nonnumerical Algorithms and Problems}
\parbox[t]{13mm}{I.2.6}
\parbox[t]{\restline}{Artificial Intelligence: Learning}
\parbox[t]{13mm}{J.3}
\parbox[t]{\restline}{Life and Medical Sciences: Biology and Genetics}
}               

\GeneralTerms{Algorithms, Experimentation, Theory}
\AdditionalKeyWords{Haplotyping, Ensemble Methods, Sequence Prediction}


%
%

\pagestyle{empty}

\makecover



\cleardoublepage
\abstractpagestart
Statistically resolving the underlying haplotype pair for a genotype
measurement is an important intermediate step in gene mapping studies,
and has received much attention recently.  Consequently, a variety of
methods for this problem have been developed.  Different methods
employ different statistical models, and thus implicitly encode
different assumptions about the nature of the underlying haplotype
structure.  Depending on the population sample in question, their
relative performance can vary greatly, and it is unclear which method
to choose for a particular sample.  Instead of choosing a single
method, we explore combining predictions returned by different
methods in a principled way, and thereby circumvent the problem of
method selection.

We propose several techniques for combining haplotype reconstructions
and analyze their computational properties.  In an experimental study
on real-world haplotype data we show that such techniques can provide
more accurate and robust reconstructions, and are useful for outlier 
detection.  Typically, the combined prediction is at least as accurate as 
or even more accurate than the best individual method, effectively 
circumventing the method selection problem.

\abstractpageend

\label{contentsendpage}


%
%

\newpage
\pagenumbering{arabic}
\setcounter{page}{1}
\thispagestyle{plain}

\renewcommand{\sectionmark}[1]{%
\markright{\sectionname
\ \thechapter.\ #1}{}%
}

\pagestyle{headings}


\section{Introduction}
Complex diseases such as Diabetes or Alzheimer's disease are often linked 
to individual genetic variations.  The analysis of genetic variation in
human populations is therefore critical for understanding individual
risk factors for such diseases.  Most of the human genome is invariant
among individuals, and it is sufficient to concentrate on small parts
of the whole genome sequence to analyze genetic variation.  Frequently
studied differences are \emph{single nucleotide polymorphisms} (SNPs),
which are single-nucleotide variations at a particular location in the
genome.  The positions in the sequence are called \emph{markers} and
the different possible values \emph{alleles}.  
A \emph{haplotype} is a sequence of SNP alleles along a chromosome, and
concisely represent the variable genetic information in that region.
In the search for DNA sequence variants which are related to common
diseases, haplotype-based approaches have become a central
theme~\cite{article/nature/TIHC06}.

Diploid human cells have two \emph{homologous} (i.e., almost
identical) copies of each chromosome.  Current practical laboratory
measurement techniques produce a \emph{genotype}---for $m$ markers, a
sequence of $m$ unordered pairs of alleles. A genotype reveals the two
alleles that are present at each marker, but not their respective
chromosome origin.  To obtain haplotypes from genotype data, this
hidden phase information has to be reconstructed.  Two alternative
approaches exist: If family trios are available, most of the
ambiguity in the haplotype pair can be resolved analytically.
Otherwise, population-based statistical methods have to be used to
estimate the haplotype pair.  Because trios are more difficult to
recruit and more expensive to genotype, the population-based
haplotyping approach is often the only cost-effective method for
large-scale studies.

The haplotyping problem has received a lot of attention recently, and
many different haplotyping methods have been
proposed~\cite{article/ajhg/StephensS05,article/ajhg/ScheetS06,
article/jcb/KimmelS05,article/bmc/EronenGT06,inproc/wabi/RastasKMU05}.
All of these methods employ different statistical models, which
reflect
different assumptions about the underlying
distribution over haplotypes in a population sample. Furthermore, the
methods offer different trade-offs in terms of reconstruction accuracy
and scaling behavior in the number of markers and individuals in the
sample.  On the other hand, the statistical properties of haplotype
``datasets'' (a particular set of markers genotyped for a particular set
of individuals) vary depending on marker spacing, sample size and
population characteristics.  In fact, some haplotyping methods have
been specifically tailored to particular dataset characteristics. For
example, the HIT system~\cite{inproc/wabi/RastasKMU05} is especially
effective for population isolates, and the HaploRec
system~\cite{article/bmc/EronenGT06} for reconstruction of large,
possibly genome-wide marker maps.
 
It is therefore unlikely that there is one haplotyping method which is
generally superior.  Instead, the relative performance of different
methods will vary depending on the characteristics of the dataset to
be haplotyped.  In contrast to other statistical modeling tasks, in
haplotyping there is typically no ``training data'' available for
which the ground truth is known.  This precludes the use of model
selection techniques such as cross-validation (although it is possible 
to use cross-validation estimates of performance on related tasks such 
as missing genotype imputation for model selection, 
see e.g.~\cite{article/ajhg/ScheetS06}).
Nevertheless, one often has to commit to just one haplotype
reconstruction in the end.  Hence, it is natural to ask whether the
predictions of the different methods could be combined in a simple way
to give more accurate and robust haplotype reconstructions without
having to know in advance which of the baseline methods performs well
on the dataset at hand.

In this paper we study how to combine haplotype reconstructions
produced by various methods.
We formulate several approaches 
for combining haplotypers, 
study the algorithmics of the 
problem, and experimentally validate that combining haplotypers
is beneficial.

\section{Population-based haplotyping}

A haplotype $h$ can be represented as a sequence of alleles $h[i]$ in markers $i=1,\ldots,m$. 
For most SNP markers, only two alternative nucleotides (alleles) occur in a population, so we can assume
$h \in \{0,1\}^m$.  
A genotype $g$ for an individual can be represented as a sequence of unordered pairs
$g[i]=\{h^1_g[i],h^2_g[i]\}$ of alleles in markers $i=1,\ldots,m$.
Hence, $g \in \{\{0,0\},\{1,1\},\{0,1\}\}^m$.  
A marker with alleles $\{0,0\}$ or $\{1,1\}$ is \emph{homozygous} whereas a marker with
alleles $\{0,1\}$ is \emph{heterozygous}.  
We denote the number of heterozygous markers by $m'$, and their positions in the haplotype
sequence by $i_1,\ldots,i_{m'}$.  

The haplotyping problem arises from the fact that while each haplotype
pair corresponds to a unique genotype, a genotype may correspond to a large
number of different haplotype pairs.
Population-based haplotyping is the task of statistically resolving this ambiguity: 
\begin{trivlist}
\item[\textbf{The haplotype reconstruction problem:$\quad$}] Given a
multiset $\mathcal{G}$ of genotypes, find for each genotype $g \in
\mathcal{G}$ the haplotypes $h^1_g$ and $h^2_g$ that have generated
$g$.
\end{trivlist}

For the rest of the paper we will denote the two individual 
haplotypes in a haplotype pair as $h^1$ and $h^2$, and use
$h$ as a shorthand to denote the pair $\{h^1,h^2\}$ when there is no ambiguity.
Furthermore, we denote a substring $s[i]s[i+1]\ldots s[i+k]$ of a string $s$ by $s[i,k]$.


For each genotype $g \in \{0,1\}^m$, there are $2^{m'-1}$ different
haplotype reconstructions.  Only one of these reconstructions is
correct, so inferring the haplotypes is clearly impossible without
additional information or assumptions.  These assumptions are
typically inspired by population genetics, and can take either a
combinatorial or a probabilistic form.  The models borrowed from
population genetics are often rather simplistic abstractions of the
complicated reality.  Furthermore, additional simplifications and
heuristics may be needed to make haplotype inference computationally
tractable.  The number of ways to combine these choices---which of the
imperfect population genetics models to build on and which
computational strategies to use---has lead to the development of a
large and diverse set of different haplotyping methods, each with
their own advantages.  The following lists just a few prominent
examples.

The currently most widely used method
\phase~\cite{article/ajhg/StephensSD01,article/ajhg/StephensD03,article/ajhg/StephensS05}
is based on quite sophisticated probabilistic models and is
computationally expensive; \fastphase~\cite{article/ajhg/ScheetS06}, a
more efficient but still almost as accurate method has been published
recently.  
Several other methods have recently been developed.
\gerbil~\cite{article/jcb/KimmelS05,article/pnas/KimmelS05}
is based on reconstructing block partitioning and resolving the
haplotypes simultaneously.
\hap~\cite{article/bioinformatics/HalperinE04} implements a method
based on imperfect phylogeny.  \hit~\cite{inproc/wabi/RastasKMU05} and
\hint~\cite{article/jbcb/KimmelS05} use HMM founder models for
haplotyping.  \haplorec~\cite{inproc/psb/EronenGT04,article/bmc/EronenGT06}
is based on variable-length Markov chains.
\spamm~\cite{inproc/pmsb/LandwehrMETM06,article/bmc/LandwehrMETM07} is an approach based on
levelwise construction of constrained Hidden Markov Models.

\section{Combining haplotypers}
\begin{comment}
In practice, genetics researchers often face the problem that
different haplotype reconstruction methods give different results and
there is no straightforward way to decide which method to choose.  In
fact, no method is the best for all populations and all individuals in
a population.
\end{comment}

In practice, genetics researchers often face the problem that
different haplotype reconstruction methods give different results and
there is no straightforward way to decide which method to choose.
Due to the varying characteristics of haplotyping datasets, it is 
unlikely that one haplotyping method is generally superior.
\begin{comment}
Different methods have different strengths and fail in different parts
of the reconstruction.  Intuitively, the multiple haplotype
reconstructions of an individual could then correct each other's
mistakes: each reconstruction might have errors here and there, but
the majority should still be correct.  This intuition can be made
precise by making probabilistic assumptions on how the reconstruction
methods err: If the errors in the reconstructions were small random
perturbations of the correct haplotype reconstruction, then taking a
majority vote (in an appropriate sense depending on the type of
perturbations) of sufficiently many reconstructions would with high
probability correct all the errors.  While such probabilistic
assumptions are not true in practice, they serve as a guideline and
motivation for the combination methods we derive next.
\end{comment}
Instead, different methods have different relative strengths and weaknesses, 
and will fail in different parts of the reconstruction.  

The promise of ensemble methods lies in ``averaging out'' those
errors, as far as they are specific to a small subset of methods 
(rather than a systematic error affecting all methods).
This intuition can be made precise by making probabilistic 
assumptions about how the reconstruction methods err:
If the errors in the reconstructions were small random 
perturbations of the true haplotype pair, taking a majority 
vote (in an appropriate sense depending on the type of 
perturbations) of sufficiently many reconstructions would with 
high probability correct all the errors.
While such probabilistic assumptions are not true in practice, 
they serve as a guideline and motivation for the combination 
methods we derive next.

\begin{comment}
Haplotypers have been combined already before.  The motivation has
been to combine results from several runs of a randomized haplotyper
to a single
answer~\cite{article/ajhg/ScheetS06,inproc/recomb/Gusfield02a}, or to
obtain a point estimate from the inferred posterior distribution on
haplotypes~\cite{article/ajhg/StephensSD01}.  However, to our
knowledge, our approach of combining unrelated haplotypers---and thus
possibly gaining the benefits of their potentially orthogonal
strengths---has not been studied before.
\end{comment}

The idea of using ensemble methods in haplotyping is not entirely new.
It is used in existing systems 
for combining results from several random restarts of a 
method~\cite{article/ajhg/ScheetS06,inproc/recomb/Gusfield02a}, or to
obtain a point estimate from an inferred posterior distribution on
haplotypes~\cite{article/ajhg/StephensSD01}.  
However, to the best of our
knowledge, our approach of combining unrelated haplotypers---and thus
gaining the benefits of their potentially orthogonal
strengths---has not been studied before.

The haplotyper combination problem can be viewed as an instance of the
general problem of finding a consensus object for a given collection
of objects.  In the simplest case the objects are individual
predictions as in ensemble methods in machine
learning~\cite{breiman96bagging,collins02adaboost,strehl02clusterensembles}.
The objects can also be more complicated structures such as
sequences~\cite{JiaoXL04,LyngsoP02,SimP03},
rankings~\cite{dwork01rank,fagin04aggregatingrankings,freund03rankboost},
clusterings~\cite{ailon05aggregating,gionis05clustering,lange05combiningpartitions},
or segmentations~\cite{inproc/kdd/MielikainenTT06}.  Although
sequential prediction has been studied a lot, there exists little work
on ensemble methods for sequence prediction.  Our approach to
haplotyper combination resembles closely the work on combining
part-of-speech taggers~\cite{Sjobergh03} to improve tagging accuracy.
However, due to the nature of haplotype data, we need more refined
strategies than simple position-wise voting.

\subsection{Problem definitions}
To combine the haplotypings suggested by $l$ given baseline haplotype
reconstruction methods, we formulate two computational problems.  We
limit ourselves to combination methods that process each individual
separately, thus enabling immediate parallelization of the combination
strategies for large populations.
\begin{problem}[Haplotyper combination]
Given the haplotype reconstructions
$\{h^1_1,h^2_1\},\ldots,\{h^1_l,h^2_l\} \subseteq \{0,1\}^m$, and a 
distance function 
$d : \{0,1\}^m \times \{0,1\}^m \to \mathbb{R}_{\geq 0}$, find:
\begin{itemize}
\item 
\HVP:\quad a reconstruction \mbox{ $\{h^1,h^2\} \subseteq \{0,1\}^m$ } minimizing the sum of distances, i.e., find
$$\{h^1,h^2\}= \argminx{ h^1_i,h^2_i\in\{0,1\}  } \sum_{i=1}^l d(\{h^1_i,h^2_i\},\{h^1,h^2\}).$$
\item 
\HSP:\quad a reconstruction \mbox{ $\{h^1_i,h^2_i\}, i\in \Set{1,\ldots,l}$ }
minimizing the sum of distances, i.e., find
$$i = \argminx{j\in \Set{1,\ldots,l}}
\sum_{i=1}^l d(\{h^1_i,h^2_i\},\{h^1_j,h^2_j\}).$$
\end{itemize}
In both cases, ties are broken arbitrarily.
\end{problem}

The difference between the Haplotyper Voting Problem (\HVP) and the
Haplotyper Selection Problem (\HSP) is that in the latter, the
solution is required to be one of the input haplotype reconstructions.
Using clustering terminology, the Haplotyper Voting Problem (\HVP)
seeks for the average haplotype reconstruction based on the input
haplotypings, whereas the Haplotyper Selection Problem (\HSP) selects
the median haplotype reconstruction as the most plausible haplotyping.
The exact meaning of average and median depends, of course, on the
properties of $d$.  Ideally, $d$ should be such that the solutions to
\HVP\ and \HSP\ can be found efficiently and are close to the unknown
true haplotypes.  We will discuss viable candidates for such $d$ later
in Section~\ref{sec:distance}.  While \HSP\ can be solved efficiently by 
brute force provided that $d$ can be computed efficiently, the computational 
aspects of \HVP\ depend heavily on $d$.  Thus, their discussion will be postponed 
to Section~\ref{sec:algo}.

\HVP\ and \HSP\ are closely related for all distance functions $d$.  A 
solution to \HSP\ is a $2$-approximation of a solution to \HVP\ and the 
solution to \HVP\ can be transformed into a $2$-approximation of a 
solution to \HSP.

\begin{proposition} \label{p:approx}
Let $d$ be a distance function between haplotype pairs satisfying the
triangle inequality.  Let \mbox{$h_1=\{h^1_1,h^2_1\}, \ldots,
h_l=\{h^1_l,h^2_l\}$} be the haplotype pairs to combine and
$h_\HVP=\{h^1_\HVP,h^2_\HVP\}$, \mbox{$h_\HSP=\{h^1_\HSP,h^2_\HSP\}$} the
optimal \HVP\ and \HSP\ solutions.  Then
\begin{enumerate}
\item
$h_\HVP$ is a feasible solution of \HVP\ and
$$\sum_{i=1}^l d(h_i,h_\HVP) \leq \sum_{i=1}^l d(h_i,h_\HSP) \leq 2\sum_{i=1}^l d(h_i,h_\HVP).$$
\item
$h_j=\argmin_{i=1,\ldots,l} d(h_i,h_\HVP)$ is a feasible solution of
\HSP\ and 
$$\sum_{i=1}^l d(h_i,h_\HSP) \leq 
\sum_{i=1}^l d(h_i,h_j) \leq 
2\sum_{i=1}^l d(h_i,h_\HSP).$$
\end{enumerate}
\end{proposition}
\begin{proof}
We have $\sum_{i=1}^l d(h_i,h_\HVP) \leq \sum_{i=1}^l d(h_i,h_\HSP)$
because
\begin{eqnarray*}
\sum_{i=1}^l d(h_i,h_\HVP)&=&\min_{h^1,h^2 \in \{0,1\}^m} \sum_{i=1}^l d(h_i,h)\\
&\leq& \min_{j \in \{1,\ldots,l\}} \sum_{i=1}^l d(h_i,h_j)\\
&=& \sum_{i=1}^l d(h_i,h_\HSP).
\end{eqnarray*}

To see that $\sum_{i=1}^l d(h_i,h_\HSP) \leq 2\sum_{i=1}^l
d(h_i,h_\HVP)$, note that there must be a haplotype pair
$h_j=\{h^1_j,h^2_j\}, j \in \{1,\ldots,l\}$ such that
$$d(h_j,h_\HVP)
\leq \frac{1}{l}\sum_{i=1}^l d(h_i,h_\HVP).$$  Hence, 
$$\sum_{i=1}^l d(h_i,h_j) 
\leq 
\sum_{i=1}^l d(h_i,h_\HVP) + l d(h_j,h_\HVP)
\leq  2\sum_{i=1}^l d(h_i,h_\HVP).$$

Similarly $\sum_{i=1}^l d(h_i,h_\HSP) \leq \sum_{i=1}^l d(h_i,h_j)$ because
$$\sum_{i=1}^l d(h_i,h_\HSP)=\min_{j \in \{1,\ldots,l\}} \sum_{i=1}^l
d(h_i,h_{j}) \leq \sum_{i=1}^l d(h_i,h_{j'})$$ for any $j' \in
\{1,\ldots,l\}$.

To see that $\sum_{i=1}^l d(h_i,h_j) \leq 2 \sum_{i=1}^l
d(h_i,h_\HSP)$, note that
$$\min_{j \in\{1,\ldots,l\}} d(h_j,h_\HVP) \leq \frac{1}{l}\sum_{i=1}^l d(h_i,h_\HVP).$$
Hence, 
\begin{eqnarray*}
\sum_{i=1}^l d(h_i,h_j) 
&\leq& \sum_{i=1}^l d(h_i,h_\HSP) +
\sum_{i=1}^l d(h_j,h_\HVP) \\
&=&
\sum_{i=1}^l d(h_i,h_\HSP) + l d(h_j,h_\HVP) \\
&\leq&
\sum_{i=1}^l d(h_i,h_\HSP) + \sum_{i=1}^l d(h_i,h_\HVP) \\
&\leq&
2\sum_{i=1}^l d(h_i,h_\HSP).
\end{eqnarray*}
\end{proof}

\subsection{Distance functions}
\label{sec:distance}

In order to define average and median haplotypings, we need to choose
a distance function $d$ for measuring the similarity between haplotype
sequences.  
To satisfy the intuition that the
solutions to \HSP\ and \HVP\ should be on average close to the baseline
haplotype reconstructions, we will focus only on a small set of 
distance measures $d$ that are reasonable candidates for measuring 
genetic distance between haplotype pairs.

\paragraph{Hamming distance and other distances induced by distances on sequences.}
The most common distance measure between sequences $s,t \in \Sigma^m$
is the Hamming distance that counts the number of disagreements
between $s$ and $t$: 
\begin{displaymath}
d_H(s,t)=\Abs{\Set{i \in \{1,\ldots,m\} : s[i]\neq t[i]}}.
\end{displaymath}
The Hamming distance is not directly applicable 
as a measure of genetic distance between individuals,
because the  haplotypes corresponding to an individual's genotype form an
unordered pair.  
To define a Hamming distance between unordered pairs of haplotypes,
let us consider haplotype pairs $\{h^1_1,h^2_1\}$ and $\{h^1_2,h^2_2\}$.  
The distance between the pairs should be zero if
the sets $\{h^1_1,h^2_1\}$ and $\{h^1_2,h^2_2\}$ are the same.  Hence,
we should try both ways to pair the haplotypes and take the one with
the smaller distance, i.e.,
\begin{align*}
d_H(\{h^1_1,h^2_1\},\{h^1_2,h^2_2\}) &= \\
&\hspace{-33mm}\min\hspace{-1mm}\Set{d_H(h^1_1,h^1_2)+d_H(h^2_1,h^2_2),d_H(h^1_1,h^2_2)+d_H(h^2_1,h^1_2)}.
\end{align*}
Note that a similar construction can be used to map any distance
function between haplotype sequences to a distance function between
pairs of haplotypings.  Furthermore, the next proposition shows that
if the distance function between the sequences satisfies the triangle 
inequality, so does the corresponding distance function for haplotype 
reconstructions.

\begin{proposition} \label{p:singletopair}
Let \mbox{$d \colon \Sigma^m \times \Sigma^m \to \mathbb{R}_{\geq 0}$} be a
distance function between sequences of length $\Sigma^m$ and let
$$d(\{h^1_1,h^2_1\},\{h^1_2,h^2_2\})=\min
\{d(h^1_1,h^1_2)+d(h^2_1,h^2_2),d(h^1_1,h^2_2)+d(h^2_1,h^1_2)\}$$ for all
\mbox{$h^1_1,h^2_1,h^1_2,h^2_2 \in \Sigma^m$}.
If $d$ satisfies the triangle inequality for comparing sequences,
i.e., $d(s,t) \leq d(s,u)+d(t,u)$ for all $s,t,u \in \Sigma^m$, then
$d$ satisfies the triangle inequality for comparing unordered pairs of
sequences $d(h_1,h_2) \leq d(h_1,h_3)+d(h_2,h_3)$ for all
$h^1_1,h^2_1,h^1_2,h^2_2,h^1_3,h^2_3 \in \Sigma^m$.
\end{proposition}
\begin{proof} 
Choose arbitrary sequences $h^1_1,h^2_1,h^1_2,h^2_2,h^1_3,h^2_3 \in
\Sigma^m$.  We show that the claim holds for them and hence for all
sequences of length $m$ over the alphabet $\Sigma$.  

Assume, without loss of generality, that
$d(\{h^1_1,h^2_1\},\{h^1_2,h^2_2\})=d(h^1_1,h^1_2)+d(h^2_1,h^2_2)$ and
$d(\{h^1_1,h^2_1\},\{h^1_3,h^2_3\})=d(h^1_1,h^1_3)+d(h^2_1,h^2_3).$

For $d(\{h^1_2,h^2_2\},\{h^1_3,h^2_3\})$ there are two cases as it is
the minimum of $d(h^1_2,h^1_3)+d(h^2_2,h^2_3)$ and
$d(h^2_2,h^1_3)+d(h^1_2,h^2_3)$.

If $d(\{h^1_2,h^2_2\},\{h^1_3,h^2_3\})=d(h^1_2,h^1_3)+d(h^2_2,h^2_3)$, then
\begin{eqnarray*}
&&d(\{h^1_1,h^2_1\},\{h^1_3,h^2_3\})+d(\{h^1_2,h^2_2\},\{h^1_3,h^2_3\})\\
&=&
d(h^1_1,h^1_3)+d(h^2_1,h^2_3)+d(h^1_2,h^1_3)+d(h^2_2,h^2_3) \\
&=&
\left[d(h^1_1,h^1_3)+d(h^1_2,h^1_3)\right]+\left[d(h^2_1,h^2_3)+d(h^2_2,h^2_3)\right] \\
&\geq&
d(h^1_1,h^1_2)+d(h^2_1,h^2_2).
\end{eqnarray*}

If $d(\{h^1_2,h^2_2\},\{h^1_3,h^2_3\})=d(h^2_2,h^1_3)+d(h^1_2,h^2_3)$, then
\begin{eqnarray*}
&&d(\{h^1_1,h^2_1\},\{h^1_3,h^2_3\})+d(\{h^1_2,h^2_2\},\{h^1_3,h^2_3\})\\
&=&
d(h^1_1,h^1_3)+d(h^2_1,h^2_3)+d(h^2_2,h^1_3)+d(h^1_2,h^2_3) \\
&=&
\left[d(h^1_1,h^1_3)+d(h^2_2,h^1_3)\right]+\left[d(h^2_1,h^2_3)+d(h^1_2,h^2_3)\right] \\
&\geq&
d(h^1_1,h^2_2)+d(h^2_1,h^1_2) \geq
d(h^1_1,h^1_2)+d(h^2_1,h^2_2).
\end{eqnarray*}

Thus, the claim holds.
\end{proof}


\paragraph{Switch distance.}
The approach of defining distance functions between haplotype
pairs based on distance functions between haplotypes has some
limitations, independently of the distance function used. 
This is because much of the variance in haplotypes 
originates from \emph{chromosomal crossover} during meiosis, which breaks up  
chromosomes and reconnects the resulting segments
to form new chromosomes for the offspring. 
The chromosome pair resulting from a crossover could be seen 
as genetically close to the original 
pair even if the  
individual sequences do not match very well.
\emph{Switch distance} is a
distance measure for haplotype pairs that takes such similarities into account.
It is defined as the number of \emph{switches} that are needed to transform a haplotype
pair to another haplotype pair with the same homozygous and
heterozygous markers.  
A switch between markers $i$ and $i+1$ for a
haplotype pair $\{h^1,h^2\}$ transforms
the pair $\{h^1,h^2\}=\{h^1[1,i]h^1[i+1,m],h^2[1,i]h^2[i+1,m]\}$ into the pair
$\{h^1[1,i]h^2[i+1,m],h^2[1,i]h^1[i+1,m]\}$.  It is easy to see that
for any pair of haplotype reconstructions corresponding to the same
genotype, there is a sequence of switches transforming one into the
other.  Thus, switch distance is well defined for the
cases we are interested in.

The switch distance has the advantage over the Hamming distance that
the order of the haplotypes in the haplotype pair does not matter in
the distance computation: the haplotype pair can be encoded uniquely
as a bit sequence consisting of just the switches between the
consecutive heterozygous markers, i.e., as a \emph{switch sequence}:
\begin{definition}[Switch sequence]
Let $h^1,h^2 \in \{0,1\}^m$ and let $i_1 < \ldots < i_{m'}$ be the
heterozygous markers in $\{h^1,h^2\}$.  The switch sequence of a
haplotype pair $\{h^1,h^2\}$ is a sequence $s(h^1,h^2)=s(h^2,h^1)=s
\in \{0,1\}^{m'-1}$ such that
\begin{displaymath}
s[j]=\left\{
\begin{array}{ll}
0 \quad & \mbox{if } h^1[i_j]=h^1[i_{j+1}] \mbox{ and } h^2[i_j]=h^2[i_{j+1}] \\
1 \quad & \mbox{if } h^1[i_j]\neq h^1[i_{j+1}] \mbox{ and } h^2[i_j]\neq h^2[i_{j+1}]
\end{array}
\right.
\end{displaymath}
\end{definition}
The switch distance between haplotype reconstructions can be defined
in terms of the Hamming distance between switch sequences as follows.
\begin{definition}[Switch distance]
Let $h_1 = \{h^1_1,h^2_1\}$ and $h_2 = \{h^1_2,h^2_2\}$ be haplotype pairs corresponding
to the same genotype.
The switch distance between the pairs is
\mbox{$d_s(h_1,h_2) 
=d_H(s(h^1_1,h^2_1),s(h^1_2,h^2_2))$}.
\end{definition}

As switch distance is the Hamming distance between the switch
sequences, the following proposition is immediate:
\begin{proposition}
The switch distance satisfies the triangle inequality.
\end{proposition}

\paragraph{$k$-Hamming distance.}
Switch distance considers only a very small neighborhood of each
marker, namely only the previous and the next heterozygous marker in
the haplotype.  On the other extreme, the Hamming distance uses the
complete neighborhood (via the $\min$ operation), i.e., the whole
haplotypes for each marker.  The intermediate cases are covered by the
following $k$-Hamming distance in which all windows of a chosen length
$k \in \{2,\ldots,m\}$ are considered.  The intuition behind the
definition is that each window of length $k$ is a potential location
for a gene, and we want to measure how close the haplotype
reconstruction $\{h^1,h^2\}$ gets to the true haplotype $\{h^1_2,h^2_2\}$
in predicting each of these potential genes.

\begin{definition}[$k$-Hamming distance]
Let $\{h^1_1,h^2_1\}$ and $\{h^1_2,h^2_2\}$ be pairs of haplotype sequences
corresponding to the same genotype with $m'$ heterozygous markers in
positions $i_1,\ldots,i_m$.  The $k$-Hamming distance $d_{k-H}$
between $\{h^1_1,h^2_1\}$ and $\{h^1_2,h^2_2\}$ is defined by
\begin{eqnarray*}
d_{k-H}(h_1,h_2)=\sum_{j=1}^{m'-k+1}  d_H(h_1[i_j,\ldots, i_{j+k-1}],h_2[i_j,\ldots,i_{j+k-1}])
\end{eqnarray*}
unless $m'<k$, in which case
  $d_{k-H}(h_1,h_2) = d_H(h_1,h_2)$.
\end{definition}
It is easy to see that $d_{2-H} = 2d_S$, and that for haplotyping pairs
with $m'$ heterozygous markers, we have $d_{m'-H} = d_{m-H} = d_H$.
Thus, the switch distance and the Hamming distance are the two extreme
cases between which $d_{k-H}$ interpolates for $k=2,\ldots,m$.

\subsection{Algorithms and complexity}
\label{sec:algo}
The \HSP\ problem is easily solved by trying out each of the $l$
reconstructions as a candidate solution, and choosing the best one.
The complexity of this straightforward strategy is roughly
$l^2$ times the time needed to evaluate the distance function $d$.
Thus, there seems to be no need for more efficient algorithms for
\HSP\ in practice.

The complexity of \HVP\ depends on $d$ in a more involved way.  As we
will show next, for $d=d_S$ a simple voting scheme gives the solution.
The rest of the distances considered in Section~\ref{sec:distance} are
more challenging.  If $d=d_{k-H}$ and $k$ is small, the solution can
be found by dynamic programming.  For $d=d_{k-H}$ with large $k$ and
$d=d_H$, we are aware of no efficient general solutions.  However, we
will outline methods that can solve most of the problem instances that
one may encounter in practice.

\paragraph{Switch distance: $d=d_S$.}
For the switch distance, the solution to \HVP\ can be found by the following
voting scheme:
\begin{itemize}
\item
  Transform the haplotype reconstructions 
  \mbox{$\{h^1_i,h^2_i\} \subseteq \{0,1\}^m$}, \mbox{$i=1,\ldots,l$} 
  into switch sequences \mbox{$s_1,\ldots,s_l \in \{0,1\}^{m'-1}$}.
\item
  Return the pair $\{h^1,h^2\}$ that shares the homozygous markers with
  the reconstructions $\{h^1_i,h^2_i\}$ and whose switch sequence 
  $s \in \{0,1\}^{m'-1}$ is defined by 
  \mbox{$s[j] = \argmaxx{b \in \{0,1\}} \Abs{\Set{j \in \{1,\ldots,m'-1\} : s_i[j]=b}}.$}
\end{itemize} 
The time complexity of this method is $O(lm)$.

\paragraph{$k$-Hamming distance: $d=d_{k-H}$}.
\label{s:hammingcomputation}
The optimal solution $h_\HVP = \{h^1_\HVP,h^2_\HVP\}$ of \HVP\ is given by 
$$h_\HVP = \argminx{\{h^1,h^2\} \subseteq \{0,1\}^m}\sum_{i=1}^l d_{k-H}(h_i,h).$$ 

The number of potentially optimal solutions is $2^{m'}$, but the
solution can be constructed incrementally based on the following
observation:
\begin{align*}
h_\HVP &= \argminx{\{h^1,h^2\}}\sum_{i=1}^l d_{k-H}(h_i,h) \\
&=\argminx{\{h^1,h^2\}}\sum_{i=1}^l \sum_{j=1}^{m'-k+1} d_H(h_i[i_j,...,i_{j+k-1}], h[i_j,...,i_{j+k-1}])
\end{align*}
Hence, the cost of any solution is a sum of terms
$$D_j(\{x,\bar{x}\})=\sum_{i=1}^l
d_H(h_i[i_j,\ldots,i_{j+k-1}],\{x,\bar{x}\}),$$
$j=1,\ldots,m'-k+1$,
$x \in \{0,1\}^{k}$ and $\bar{x}$ denotes the complement of $x$.
There are $(m'-k+1)2^{k-1}$ such terms.
Furthermore,
the cost of the optimal solution can be computed by
dynamic programming using the recurrence relation
$$T_j(\{x,\bar{x}\}) \hspace{-0.5mm}=\hspace{-0.5mm}
\left\{\hspace{-1mm}
\begin{array}{ll}
0 & \hspace{-1mm}\mbox{if } j=0 \\
D_j(\{x,\bar{x}\}) + \hspace{-1.5mm}\minx{b \in \{0,1\}}T_{j-1}(\{bx,\overline{bx}\}) & \mbox{\hspace{-1mm}if } j>0
\end{array}
\right.$$
Namely, the cost of the optimal solution is 
$\min_{x \in \{0,1\}^k} T_{m'}(\{x,\bar{x}\})$ and the optimal
solution itself can be reconstructed by backtracking the path that
leads to this position.  
The total time complexity for 
finding the optimal solution using dynamic programming is
$\Oh{lm+2^kkl(m'-k)}$: the heterozygous markers can be detected and
the data can be projected onto them in time $\Oh{lm}$, and the optimal
haplotype reconstruction for the projected data can be computed in
time $\Oh{2^kkl(m'-k)}$. So the problem is fixed-parameter
tractable\footnote{A problem is called fixed-parameter tractable in a
parameter $k$, if the running time of the algorithm is $f(k)\Oh{n^c}$
where $k$ is some parameter of the input and $c$ is a constant (and
hence not depending on $k$.)  For a good introduction to
fixed-parameter tractability and parameterized complexity,
see~\cite{Flum06}.} in $k$.

\paragraph{Hamming distance: $d=d_H$.}  
An ordering $(h^1,h^2)$ of an optimal solution $\{h^1,h^2\}$ to \HVP\
with Hamming distance determines an ordering of the unordered input
haplotype pairs $\{h^1_1,h^2_1\},\ldots,\{h^1_l,h^2_l\}$.  This
ordering can be represented by a binary vector $o = (o_1,\ldots,o_l)
\in \{0,1\}^l$ that states for each $i = 1,\ldots,l$ that the ordering
of $\{h^1_i,h^2_i\}$ is $(h^{1+o_i}_i,h^{2-o_i}_i)$.  Thus,
$o_i=\argmin_{b \in \{0,1\}} d_H(h^1,h^{1+b}_i)$, where ties are
broken arbitrarily.

If the ordering $o$ is known and $l$ is odd, the optimal haplotype
reconstruction can be determined in time $\Oh{lm}$ using the formulae
\begin{equation}
\label{eq:ball1}
h^1[i]=\argmax_{b \in \{0,1\}} 
=\Abs{\Set{j \in \{1,\ldots,l\} : h^{1+o_j}_j[i]=b}}
\end{equation}
and 
\begin{equation}
\label{eq:ball2}
h^2[i]=\argmax_{b \in \{0,1\}} 
=\Abs{\Set{j \in \{1,\ldots,l\} : h^{2-o_j}_j[i]=b}}.
\end{equation}
Hence, solving \HVP\ is polynomial-time equivalent to the task of
determining the ordering vector $o$ corresponding to the best haplotype
reconstruction $\{h^1,h^2\}$.

The straightforward way to find the optimal ordering is to evaluate
the quality of each of the $2^{l-1}$ non-equivalent orderings.  The
quality of a single ordering can be evaluated in time $\Oh{lm}$.
Hence, the \HVP\ problem can be solved in total time
$\Oh{lm+2^{l}lm'}$.  The runtime can be reduced to $\Oh{lm+2^{l}m'}$
by using Gray codes~\cite{Savage97} to enumerate all bit vectors $o$
in such order that consecutive bit vectors differ only by one bit.
Hence, the problem is fixed-parameter tractable in $l$ (i.e., in the
number of methods).
If $l$ is large, however, a more clever strategy is needed.  We are
unaware of a tractable efficient general solution and suspect that
\HVP\ for $d=d_H$ is NP-complete in general.  However, we have
efficient solutions to two special cases of practical relevance:

\noindent\emph{Small number of heterozygous markers.}  If the number
of heterozygous positions $m'$ is small, we can simply enumerate all
the $2^{m'-1}$ non-equivalent possible solutions to the problem, and
pick the optimal one from among them.  The time complexity of this
approach is $O(2^{m'} lm')$.  Thus, the problem is fixed-parameter
tractable also in $m'$ (the number of heterozygous markers).

\noindent\emph{All reconstructions close to the optimal solution for
\HVP.}  Fixing an ordering to any one of the input haplotype
reconstructions $\{h^1_i,h^2_i\}$ induces an ordering to the remaining
input haplotypes.  This ordering can be used to compute a solution to
\HVP\ through equations~\ref{eq:ball1} and~\ref{eq:ball2}.  The next
proposition shows that the solution obtained in this way is provably optimal
if all input haplotype reconstructions are within $m'/2$ of the
optimal solution $\{h^1_\HVP,h^2_\HVP\}$ of \HVP.
\begin{proposition}
If $d_H(\{h^1_i,h^2_i\},\{h^1_\HVP,h^2_\HVP\})<m'/2$ for each $i \in
\{1,\ldots,l\}$, then the ordering induced by any of the input
haplotype pairs is equivalent to the ordering corresponding to the
optimal solution to \HVP.
\end{proposition}
\begin{proof}
By assumption, 
$d_H(h^{1+o_i}_i,h^1_\HVP) < m'/4$ for each $i=1,\ldots,l$ for one of the choices $o_i \in \{0,1\}$.  Then 
$$d_H(h^{1+o_i}_i,h^{1+o_j}_j) \leq d_H(h^{1+o_i}_i,h^1_\HVP) +  d_H(h^{1+o_j}_j,h^1_\HVP) < m'/4+ m'/4 = m'/2.$$

Thus, if we use $(h^{1+o_i}_i,h^{2-o_i}_i)$ as a reference point, the
induced ordering for the haplotypes will be the same $o$ that is
induced by using $(h^1_\HVP,h^2_\HVP)$ as a reference point.
Switching the ordering in the reference point to
$(h^{2-o_i}_i,h^{1+o_i}_i)$ will induce the equivalent ordering $1-o$.
\end{proof}

\section{Experiments}
To investigate the haplotype reconstruction combination problem
empirically, real-world genotype data was phased with different
haplotyping systems and their reconstructions evaluated.
The data was obtained from three sources:  a collection 
of datasets from the \textbf{Yoruba} population in Ibadan, 
Nigeria~\cite{article/nature/TIHC06}, the well-known dataset 
derived from a European population of \textbf{Daly} et 
al.~\cite{article/naturegenetics/DalyRSHL01}, and samples from the 
recently published D-\textbf{HaploDB} haplotype 
database~\cite{article/nar/HigasaMKTH07} derived from a Japanese population.
For the Yoruba and Daly data, true haplotype pairs were inferred 
from family trios. Furthermore, the nontransmitted parental chromosomes of each trio were
combined to form additional artificial haplotype pairs. 
For the HaploDB dataset, definite haplotypes were
determined from complete hydatidiform moles (CHMs).
The $74$ available CHMs haplotypes were paired to form $37$ diploid individuals. 

For all datasets markers with minor allele frequency of less 
than $5\%$ and genotypes with more than $15\%$ missing values were removed.
For the Yoruba population, information on $3.8$ million SNPs spread 
over the whole genome is available.  We sampled $100$ sets of 
$100$ markers each from distinct regions on 
chromosome $1$.  There are $60$ individuals 
in these datasets after preprocessing as described above, with an average 
fraction of missing values of $3.6\%$ and $32.2\%$ heterozygous markers.  
For the Daly dataset, there is information on $103$ markers and $174$ 
individuals available after data preprocessing, the average fraction of 
missing values is $7.9\%$ and the average fraction of heterozygous markers 
is $30.6\%$. In HaploDB, a genome-wide set of $281\,439$ SNP markers 
is available, from which we sampled  $100$ sets of 
$100$ markers each from distinct regions on chromosome $1$.
The average fraction of missing values is
$3.1\%$ and the fraction of heterozygous markers is $39.9\%$.
All datasets were phased with each of the following $6$ publicly
available haplotyping systems, yielding $6$ different reconstructed
haplotype pairs for every genotype: \phase\ version
2.1.1.~\cite{article/ajhg/StephensS05}, \fastphase\ version
1.1.~\cite{article/ajhg/ScheetS06}, \gerbil\ as included in \gevalt\
version 1.0.~\cite{article/jcb/KimmelS05},
\hit~\cite{inproc/wabi/RastasKMU05}, \haplorec\ version
2.0.~\cite{article/bmc/EronenGT06} and \spamm\ version
1.0.~\cite{inproc/pmsb/LandwehrMETM06}.  All methods were run using
their default parameters.

\begin{table}
\centering
\caption{Switch (top-right triangle) and Hamming (bottom-left triangle) 
distances between the truth and the baseline methods for the Daly 
dataset. \label{t:dalybase}}
\begin{tabular}{@{}r|rrrrrrr@{}}
\toprule
& \trio & \fastphases & \hits & \spamms & \haplorecs & \gerbils & \phases \\
\midrule
\trio        & - & 105 & 121 & 127 & 131 & 132 & 145 \\
\fastphases  & 480 & - & 82 & 82 & 104 & 85 & 118 \\
\hits        & 514 & 414 & - & 88 & 116 & 103 & 146 \\
\spamms      & 510 & 434 & 508 & - & 118 & 117 & 122 \\
\haplorecs   & 736 & 676 & 784 & 716 & - & 119 & 150 \\
\gerbils     & 568 & 478 & 522 & 546 & 810 & - & 143 \\
\phases      & 654 & 590 & 728 & 650 & 850 & 718 & - \\
\bottomrule
\end{tabular}
\begin{comment}
\medskip

\begin{tabular}{@{}r|rrrrrrr@{}}
\toprule
& \trio & \fastphases & \hits & \spamms & \haplorecs & \gerbils & \phases \\
\midrule
\trio & 
     - & 480 & 514 & 510 & 736 & 568 & 654 \\
\fastphases & 
   480 &   - & 414 & 434 & 676 & 478 & 590 \\
\hits & 
   514 & 414 &   - & 508 & 784 & 522 & 728 \\
\spamms& 
   510 & 434 & 508 &   - & 716 & 546 & 650 \\
\haplorecs & 
   736 & 676 & 784 & 716 &   - & 810 & 850 \\
\gerbils &
   568 & 478 & 522 & 546 & 810 &   - & 718 \\
\phases &
   654 & 590 & 728 & 650 & 850 & 718 &   - \\
\bottomrule
\end{tabular}
\end{comment}
\end{table}

Let us first consider how the reconstructions produced by the baseline
methods differ on the Daly dataset.  Table~\ref{t:dalybase} shows the
switch and Hamming distances between the different haplotype
reconstructions, including the reconstructions inferred from the
family trios (\trio) as the ground truth, and the methods \fastphase\
(\fastphases), \hit\ (\hits), \spamm\ (\spamms), \haplorec\
(\haplorecs), \gerbil\ (\gerbils), and \phase\ (\phases).
The \fastphase\ system  clearly has the smallest 
reconstruction error with respect to switch and Hamming distances on 
the Daly dataset.  While
the accuracy performance of the other methods is worse, the distances
between all the methods are of the same order of magnitude.  
This indicates that it makes sense to try to combine the haplotypers.  

We tested the haplotyper selection and voting techniques using the
Daly, Yoruba and HaploDB datasets. 
As it is not clear which combination of selection/voting and internal 
distance measure 
(switch distance, Hamming distance, $k$-Hamming distance) yields 
best results, systematic experiments using all different combinations 
were performed. 
The quality of the resulting reconstructions is measured by switch distance
only, as this is the standard way of measuring the quality of 
reconstructions in haplotyping experiments.

The main goals of the experimental study are as follows.  First, the
goal is to evaluate whether the simple combination approaches like
selection and voting can be used to find a more robust solution when
the best-performing method is not known.  Second, the goal is to see
whether the combination methods improve over the baseline methods when
using different subsets of the baseline methods.  For this purpose, we
consider leaving out one of the baseline methods \phase\ (the most
accurate on the Yoruba and HaploDB datasets on average and on the
HaploDB dataset), \fastphase\ (most accurate on the Daly dataset), and
\gerbil\ (slow and least accurate on all datasets), and also leaving
out all three of them simultaneously.  The results using these subsets
are representative of results for other subsets we experimented with
but do not report on here.  Third, the goal is to find out how the
haplotyper selection results compare to haplotyper voting results and
how the different distance functions affect the quality of the
solutions.

The results for the baseline methods are summarized in Table~\ref{t:baseerrors}
and results for the combination methods in Table~\ref{t:selectcomberrors} and Table~\ref{t:votecomberrors}.
Let us first consider the Daly dataset.  The best baseline 
method is \fastphase, resulting in $105$ switch errors.  The selection 
and voting methods applied to the set of all baseline methods produce 
results comparable to \fastphase, and are consistently better than the 
haplotype reconstructions produced by any other baseline method.  Thus, 
by employing the haplotyper combination approach, we can achieve performance 
comparable to the best baseline method without having to know which of 
the baseline methods is best in advance.  Leaving out one of the methods 
\phase, \fastphase, or \gerbil\ has no significant effects on the results 
of the combination methods.  Thus, the combination methods seem to be 
quite robust against small perturbations of the set of baseline methods, 
even if they lead to the exclusion of the best performing method.  If \phase,
\fastphase, and \gerbil\ are left out simultaneously, the results
degrade below the level of \fastphase, but are still better than those
of any other method.

\begin{table}
\centering
\caption{The total switch error between true haplotypes and the
haplotype reconstructions over all individuals for the baseline
methods.  For Yoruba and HaploDB, the reported numbers are averages 
over the 100 datasets.
\label{t:baseerrors}}
\begin{tabular}{l|r|r|r}
\toprule
\multicolumn{1}{c}{Method} & \multicolumn{1}{c}{Daly} & \multicolumn{1}{c}{Yoruba} & \multicolumn{1}{c}{HaploDB} \\
\midrule
\phase & 145 & 37.61 &  108.36\\
\fastphase & 105 & 45.87 &  110.45\\
\spamm & 127 & 54.69 &  120.29\\
\haplorec & 131 & 56.62 &  130.28\\
\hit & 121 & 73.23 &  123.95\\
\gerbil & 132 & 75.05 &  134.22\\
\bottomrule
\end{tabular}
\end{table}

\begin{table}
\centering
\caption{The total switch error between true haplotypes and the
haplotype reconstructions over all individuals for the haplotyper
selection methods for different combinations of baseline haplotypers.  
For Yoruba and HaploDB, the reported numbers are averages over the 100 datasets.
\label{t:selectcomberrors}}
\begin{tabular}{ll|r|r|r}
\toprule

\multicolumn{5}{c} {Haplotyper Selection} \\

\midrule
\multicolumn{1}{l}{Methods} &
\multicolumn{1}{l}{Distance} &
 \multicolumn{1}{c}{Daly} & 
 \multicolumn{1}{c}{Yoruba} &
 \multicolumn{1}{c}{HaploDB} \\

\midrule
all methods & $d_s$ & 103 & 37.67 & 103.43 \\
 & $d_{3-H}$ & 103 & 38.29  & 104.10 \\
& $d_{4-H}$ & 103 & 38.41  & 104.52 \\
 & $d_{5-H}$ & 105 & 38.35  & 104.76 \\
 & $d_H$ & 107 & 40.14  & 110.84 \\
\midrule

w/o \phase & $d_s$ & 106 & 43.58 & 107.16 \\
 & $d_{3-H}$ & 107 & 43.42 &  107.40 \\
 & $d_{4-H}$ & 109 & 43.99 &  108.55 \\
 & $d_{5-H}$ & 107 & 44.16 &  108.55 \\
 & $d_H$ & 102 & 48.36 &  117.00 \\

\midrule

w/o \fastphases & $d_s$ & 108 & 40.00 & 105.06 \\
 & $d_{3-H}$ & 105 & 40.13  & 105.74 \\
 & $d_{4-H}$ & 110 & 40.82  & 106.91 \\
 & $d_{5-H}$ & 114 & 41.59  & 107.32 \\
 & $d_H$ & 115 & 45.30  & 116.27 \\

\midrule

w/o \gerbil & $d_s$ & 103 & 38.47 & 103.91 \\
 & $d_{3-H}$ & 105 & 38.53  & 104.69 \\
 & $d_{4-H}$ & 104 & 38.98  & 105.52 \\
 & $d_{5-H}$ & 118 & 39.07  & 106.05 \\
 & $d_H$ & 113 & 42.46  & 111.57 \\

\midrule

w/o \phases, \fastphases, \gerbils & $d_s$ & 116 & 47.94 & 113.95 \\
 & $d_{3-H}$ & 108 & 47.87  & 114.39 \\
 & $d_{4-H}$ & 104 & 48.48  & 115.57 \\
 & $d_{5-H}$ & 116 & 48.66  & 116.57 \\
 & $d_H$ & 117 & 53.47  & 122.61 \\
\bottomrule
\end{tabular}
\end{table}

\begin{table}
\centering
\caption{The total switch error between true haplotypes and the
haplotype reconstructions over all individuals for the haplotyper
voting methods for different combinations of baseline haplotypers.  
For Yoruba and HaploDB, the reported numbers are averages over the 100 datasets.
\label{t:votecomberrors}}
\begin{tabular}{ll|r|r|r}
\toprule

\multicolumn{5}{c} {Haplotyper Voting} \\

\midrule
\multicolumn{1}{l}{Methods} &
\multicolumn{1}{l}{Distance} &
 \multicolumn{1}{c}{Daly} & 
 \multicolumn{1}{c}{Yoruba} &
 \multicolumn{1}{c}{HaploDB} \\

\midrule
all methods & $d_s$ & 104 & 39.86 & 103.06 \\
 & $d_{3-H}$ & 107 & 39.15 &  102.24 \\
 & $d_{4-H}$ & 107 & 40.08 &  104.00 \\
 & $d_{5-H}$ & 107 & 39.56 &  104.29 \\
 & $d_H$ & 106 & 51.07 & 134.16 \\
\midrule

w/o \phase & $d_s$ & 107 & 43.18 & 105.68 \\
 & $d_{3-H}$ & 107 & 43.15 & 106.41 \\
 & $d_{4-H}$ & 114 & 43.67 & 107.14 \\
 & $d_{5-H}$ & 107 & 44.14 & 107.67 \\
 & $d_H$ & 105 & 50.29 & 119.99\\ 
\midrule

w/o \fastphases & $d_s$ & 109 & 39.71 & 103.77 \\
 & $d_{3-H}$ & 107 & 39.92 & 104.26 \\
 & $d_{4-H}$ & 106 & 40.79 & 105.42 \\
 & $d_{5-H}$ & 112 & 41.34 & 105.78 \\
 & $d_H$ & 117 & 47.59 & 119.19 \\
\midrule

w/o \gerbil & $d_s$ & 105 & 38.27 & 102.76 \\
 & $d_{3-H}$ & 104 & 38.19  & 103.38 \\
 & $d_{4-H}$ & 104 & 38.70  & 104.39 \\
 & $d_{5-H}$ & 112 & 38.93  & 104.62 \\
 & $d_H$ & 110 & 43.91 & 114.93 \\
\midrule

w/o \phases, \fastphases, \gerbils & $d_s$ & 112 & 46.28 & 110.58 \\
 & $d_{3-H}$ & 109 & 46.58  & 110.99 \\
 & $d_{4-H}$ & 107 & 48.09 &  113.25 \\
 & $d_{5-H}$ & 111 & 48.60  & 113.91 \\
 & $d_H$ & 114 & 53.92 & 122.99 \\
\bottomrule
\end{tabular}
\end{table}

The results on the Yoruba datasets follow a similar pattern, except
that now \phase\---the baseline method with worst performance on the
Daly dataset---is the best on average.  The combination methods
provide solutions comparable to those of the best method (\phase) and better than 
those of any other baseline method.
When only subsets of the baseline methods are used, the
performance of the combination methods drops, but not significantly
unless \phase, \fastphase, and \gerbil\ are left out simultaneously.

On the HaploDB dataset the advantage of using combination 
methods is even more evident.  The best baseline method \phase\ is
clearly outperformed by all combination methods except voting and
selection with Hamming distance.  The results are only slightly
degraded when \phase, \fastphase, or \gerbil\ is left out of the
ensemble, and when they are all left out simultaneously, the performance
of the combination strategies is still significantly better than that 
of the best remaining baseline method \spamm.  

In summary, our results indicate that using the haplotyper combination 
approach sometimes significantly increases the haplotyping accuracy, and
never significantly decreases the accuracy in comparison to the best baseline 
method.  Of course, in practice the identity of the best baseline method is 
not known and changes from dataset to dataset.  In a more realistic 
comparison to the baseline method that does best on average on all the 
datasets (\fastphase), all the proposed haplotyper combination methods 
are clearly more accurate on average.  Hence, our experiments suggest that 
it is indeed better to combine the predictions of all the baseline methods 
than to (blindly) choose and use any one of them.

In general, combination methods using switch distance as the distance
function tend to produce most accurate results.  This suggests that
the errors of the baseline methods resemble random switches rather
than random single nucleotide mutations.  The performance of different
distance functions also depends on the density of the used marker map.
For dense marker maps, larger windows are beneficial, whereas
in sparse maps considering dependencies between consecutive
markers probably suffices. Furthermore, the selection methods seem to
perform slightly better than voting methods.  A potential explanation is
that the median haplotype reconstruction is more tolerant to random
errors in the baseline methods than the mean haplotype reconstruction.
Further analysis is needed in order to fully understand the differences 
between the combination methods, but it seems safe to conclude that 
haplotyper selection with switch distance is the best choice (among
combination methods and baseline methods) at least when no additional 
information abut the problem at hand is available.

The computational price for the potential improvements in accuracy is 
the added effort of first running all the baseline methods and
then solving the \HVP\ or \HSP\ problem.  This may be a problem if some 
of the baseline methods are very slow.  In such cases, we suggest the 
strategy of computing the predictions of as many baseline methods as 
time constraints permit, and combining the resulting reconstructions 
using one of the combination methods.  The running times of the baseline 
methods vary greatly, so running, e.g., all but the slowest baseline 
method may well be much more efficient than running the slowest method 
alone. 

A major computational difficulty with Hamming voting is that the basic
method for computing it scales exponentially in the number of
haplotype reconstructions per individual.  In
Section~\ref{s:hammingcomputation} we showed that if the number of
heterozygous markers is small or the relative Hamming score of the
solution is at most $1/2$ then the ordering of the haplotypes in the
pairs can be determined efficiently.  Figure~\ref{f:dalyHamming}
illustrates that in the Daly dataset most of the individuals have
either very small number of heterozygous markers or small relative
Hamming error.  This supports the hypothesis that the Hamming voting
problem can be solved sufficiently efficiently in practice even with a
larger number of baseline methods.

\begin{figure}
\includegraphics[width=\columnwidth]{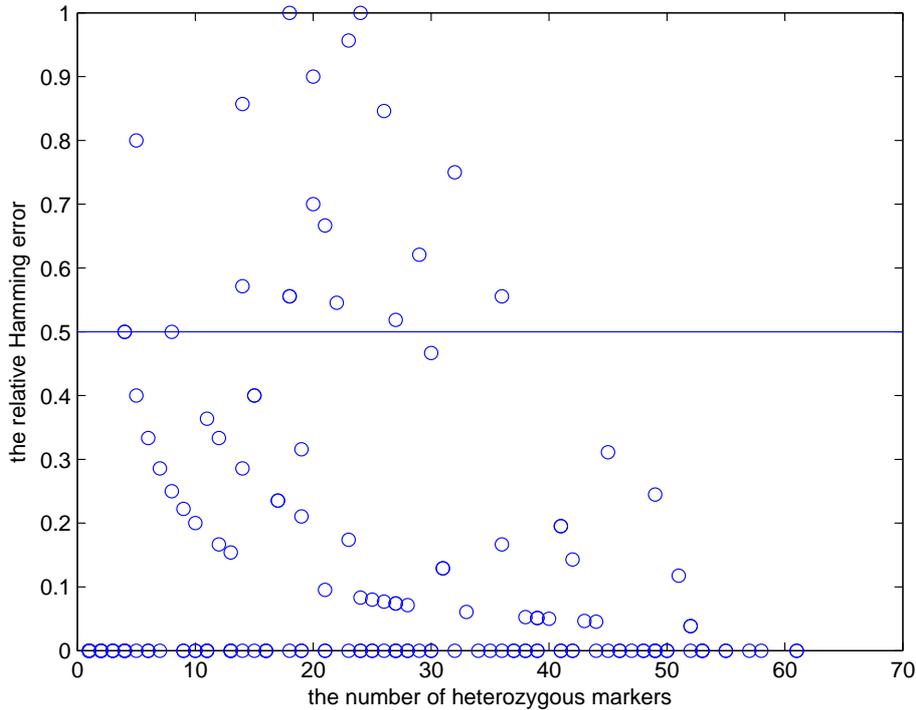}
\caption{The relative Hamming distance vs the number of heterozygous
markers in the Daly dataset.  The plot shows that the relative error
being higher than $1/2$ is mainly the problem of individuals with
small number of heterozygous markers. The fraction of individuals with
the relative hamming error at least $1/2$ is $20/174 \approx 11.5\%$.
 \label{f:dalyHamming}}
\end{figure}

Depending on the combination method there can be multiple solutions
that have the same score but different distance to the ground truth.
In Table~\ref{t:selectcomberrors} and Table~\ref{t:votecomberrors}
ties are broken by selecting the 
solution with optimal score that was found first.
The rule for breaking ties may
have a significant effect on the final accuracy:  For example, when 
applying switch voting to all the $6$ methods on the Daly dataset, 
there are a total of $35$ ties.  Thus, depending on how ties are
broken, the error may be anywhere between $89$ and $124$.  
Breaking the ties randomly results on average in $106.5$ errors, which is quite
close to the result $105$ obtained by our arbitrary tie breaking.
The situation with other combination methods and datasets is similar.
Better results might be obtained by more sophisticated tie breaking 
rules, e.g., by following the overall leader \fastphase\ in case of 
ties.  We leave the exploration of such advanced tie breaking 
rules for future work.

\begin{figure}
\includegraphics[width=\columnwidth]{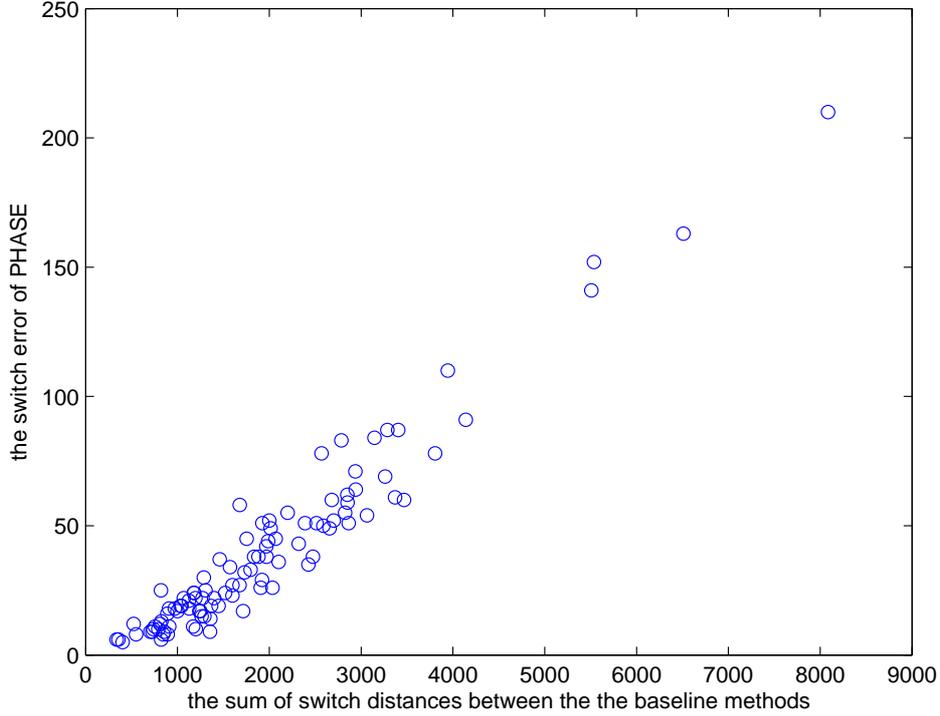}
\caption{The switch error of \phase\ vs the sum of the switch
distances between the baseline methods for Yoruba datasets.  Each
point corresponds to one of the Yoruba datasets, x-coordinate being the
sum of distances between the reconstructions obtained by the baseline
methods, and y-coordinate corresponding to the switch errors of the
reconstructions by \phase.
\label{f:distnswitch}}
\end{figure}

When the experimental results are analyzed in detail at the
level of individuals, it turns out that the combination methods tend
to fail only when many of the baseline methods perform rather badly.
Even though it is hard to recover the true haplotypes in such cases,
the fact that the baseline methods are in wide disagreement can be
used to identify such problematic individuals.  We have observed that
the sum of distances between the baseline haplotype reconstructions has
very high correlation (between $0.95$ and $0.99$) with the error in
the final reconstructions for the individual.  Figure~\ref{f:distnswitch}
illustrates this for \phase\ using the Yoruba datasets. 
This indicates that combining haplotypers can also be a strong method 
for outlier detection, which is helpful for removing probably incorrectly
haplotyped individuals from further consideration.

\section{Conclusions}
Haplotype reconstruction is an important intermediate task in the
study of genetic variations in human populations.  Various techniques
to reconstruct haplotypes from measurable genotype data have been
proposed.  Different methods typically return substantially different
reconstructions, and there seems to be no method that is generally
superior on all datasets.


To overcome these difficulties, we have studied the problem of
combining haplotypers in order to improve the haplotype
reconstructions.  More specifically, we have considered two variants
of the problem: haplotyper voting, where the goal is to find a
consensus haplotype reconstruction given multiple haplotype
reconstructions, and haplotyper selection, where the goal is to find
the best haplotype reconstruction for each individual.  We have
developed algorithms for using various internal distance functions.
The experiments show that combining haplotypers provides improvements
over the average performance of the haplotype reconstruction methods,
and the reconstruction quality is even comparable to or better than
the best method for each dataset.  Hence, using the combination
methods virtually never degrades the performance, and sometimes gives
clear advantages on accuracy in comparison to the best baseline
method.

According to the experiments, haplotyper selection with switch distance 
is consistently close to the best combination method.  Thus, the original 
problem of having to choose a baseline methods has not been lifted to
an analogous problem of choosing a combination method as haplotyper
selection with switch distance seems to be a good choice always.

Combining haplotypers opens many avenues to improved techniques for
haplotype reconstruction.  First, an obvious direction of refinements
would be to use more complex combinator functions.  For example, there
could be a-priori knowledge about the performance of the methods on
some part of the data or on other, similar datasets.  This knowledge
could be used to reweight methods in the combination algorithms, or
pursue more complex prediction approaches such as decision trees or
support vector machines.  Such approaches would be especially useful
for combining a large number of reconstructions of varying quality.
Second, the different methods could be used to guide haplotype
reconstruction techniques, e.g., to detect potentially problematic
regions of the data where the reconstruction model should be refined.
Third, assessing the quality of haplotype reconstructions by combining
haplotypers is a promising direction.  Often it will be acceptable to
discard part of the reconstructed haplotypes or markers to avoid
errors.  Our preliminary results suggest that multiple haplotype
reconstructions could be used to detect individuals which are likely
to be haplotyped erroneously and even problematic regions of the
marker map.  Fourth, more refined measures of the reconstruction
quality are also of interest, for example modeling the dependence
structure of the markers in more detail or taking genetic background
information into account.  Finally, we intend to evaluate the approach
with further genotype datasets with known haplotypes as they become
available.

\bibliographystyle{alpha}
\bibliography{hapcomb}

\label{lastpage}

\end{document}